\documentclass[a4paper]{article}

\usepackage{amsmath}
\usepackage{amssymb}
\usepackage{amsthm}
\usepackage[english]{babel}
\usepackage{caption} 
\usepackage{color}
\usepackage{enumerate}
\usepackage{float}
\usepackage{framed}
\usepackage{graphicx}
\usepackage{hyperref}
\usepackage{relsize} 
\usepackage{subcaption} 
\usepackage{xspace}
\usepackage{natbib} 
\usepackage{verbatim}

\usepackage[ruled,algosection,linesnumbered]{algorithm2e}

\renewcommand{\cite}[1]{\citep{#1}} 
\numberwithin{table}{section}
\numberwithin{figure}{section}
\numberwithin{equation}{section}

\newcommand{\refline}[1]{Line~\ref{#1}}

\theoremstyle{definition}
\newtheorem{theorem}{Theorem}[section]
\newtheorem{example}[theorem]{Example}
\newtheorem{definition}[theorem]{Definition}
\newtheorem{proposition}[theorem]{Proposition}
\newtheorem{remark}[theorem]{Remark}

\newcommand{\fname}[1]{\mathit{#1}} 
\renewcommand{\qed}{\hfill$\square$} 

\newcommand{\nat}{\mathbb{N}}
\newcommand{\ints}{\mathbb{Z}}

\newcommand{\set}[1]{\{#1\}}

\newcommand{\abs}[1]{\fname{abs}\left(#1\right)}
\newcommand{\powset}[1]{\mathcal{P}(#1)}

\newcommand{\intrange}[2]{\set{#1\,..\,#2}}

\newcommand{\Aplus}{+1}
\newcommand{\Bplus}{+5}
\newcommand{\Bminus}{-10}

\newcommand{\actions}{A}
\newcommand{\act}{a}
\newcommand{\task}{T}
\newcommand{\tasktup}{(\states,\startstates,\actions,\features,\tr,\ff,\avs)}
\newcommand{\states}{S}
\newcommand{\startstates}{B}
\newcommand{\tr}{\delta} 
\newcommand{\ff}{\varphi} 
\newcommand{\f}{f} 
\newcommand{\st}{s}
\newcommand{\avs}{\Omega} 
\newcommand{\features}{F} 

\newcommand{\X}[2]{#1_{#2}}
\newcommand{\taskX}[1]{\X{\task}{#1}}
\newcommand{\actionsX}[1]{\X{\actions}{#1}}
\newcommand{\statesX}[1]{\X{\states}{#1}}
\newcommand{\startstatesX}[1]{\X{\startstates}{#1}}
\newcommand{\trX}[1]{\X{\tr}{#1}}
\newcommand{\ffX}[1]{\X{\ff}{#1}}
\newcommand{\avsX}[1]{\X{\avs}{#1}}
\newcommand{\featuresX}[1]{\X{\features}{#1}} 
\newcommand{\tasktupX}[1]{(\statesX{#1},\startstatesX{#1},\actionsX{#1},\allowbreak\featuresX{#1},\allowbreak\trX{#1},\ffX{#1},\avsX{#1})}

\newcommand{\view}{V} 
\newcommand{\pol}{\pi}
\newcommand{\prop}[2]{\fname{act}(#1,#2)}
\newcommand{\propX}[3]{\fname{act}_{#1}(#2,#3)} 

\newcommand{\aLearn}{A-learning}
\newcommand{\mem}{P} 
\newcommand{\memfix}{{P^*}} 

\newcommand{\jump}[1]{{}\xrightarrow{#1}}

\newcommand{\g}[1]{\text{.#1}} 

\newcommand{\grid}{\mathcal{G}}
\newcommand{\gridwidth}{\mathit{Width}}
\newcommand{\gridheight}{\mathit{Height}}
\newcommand{\gridtargets}{\mathit{Targets}}
\newcommand{\gridstarts}{\mathit{Starts}}
\newcommand{\gridtime}{\tau}
\newcommand{\taskof}[1]{\mathit{task}(#1)}

\newcommand{\mandist}{d_\text{1}} 
\newcommand{\mannorm}[1]{\left|#1\right|_1}
\newcommand{\gridmove}{\mathit{move}}

\begin{document}

\title{On Avoidance Learning with Partial Observability}

\author{
    Tom~J.~Ameloot\thanks{T.J.~Ameloot is a Postdoctoral Fellow of the Research Foundation -- Flanders (FWO).}
    \\    
    {\small Hasselt University and transnational University of Limburg}}
\date{}

\maketitle{}

\begin{abstract}
We study a framework where agents have to avoid aversive signals.
The agents are given only partial information, in the form of features that are projections of task states. 
Additionally, the agents have to cope with non-determinism, defined as unpredictability on the way that actions are executed.
The goal of each agent is to define its behavior based on feature-action pairs that reliably avoid aversive signals.
We study a learning algorithm, called \aLearn, that exhibits fixpoint convergence, where the belief of the allowed feature-action pairs eventually becomes fixed. 
\aLearn\ is parameter-free and easy to implement.
\end{abstract}

\setcounter{tocdepth}{2}
\tableofcontents

\section{Introduction}

The main aim of this paper is to let agents solve tasks by ultimately avoiding aversive signals forever. This approach entails an interesting and perhaps quite strong guarantee on the agent performance.
The motivation is partly to understand how animals are successful in solving problems, like navigation~\cite{geva-sagiv_2015}, with limited sensory information and unpredictable effects in the environment. The animal should find food or return home before it gets lost or becomes exhausted.

We study a general framework in which agents need to avoid problems in tasks. If the agent encounters a problem, an aversive signal is received. This way the agent could learn to avoid the problem, by avoiding the usage of actions and action-sequences that lead to aversive signals.
The general idea is sketched in Figure~\ref{fig:aversive}.
Before we discuss our approach, we first briefly discuss two important ingredients of the framework, namely, partial information and non-determinism.

\begin{figure}
    \begin{center}
    \includegraphics[width=0.4\textwidth]{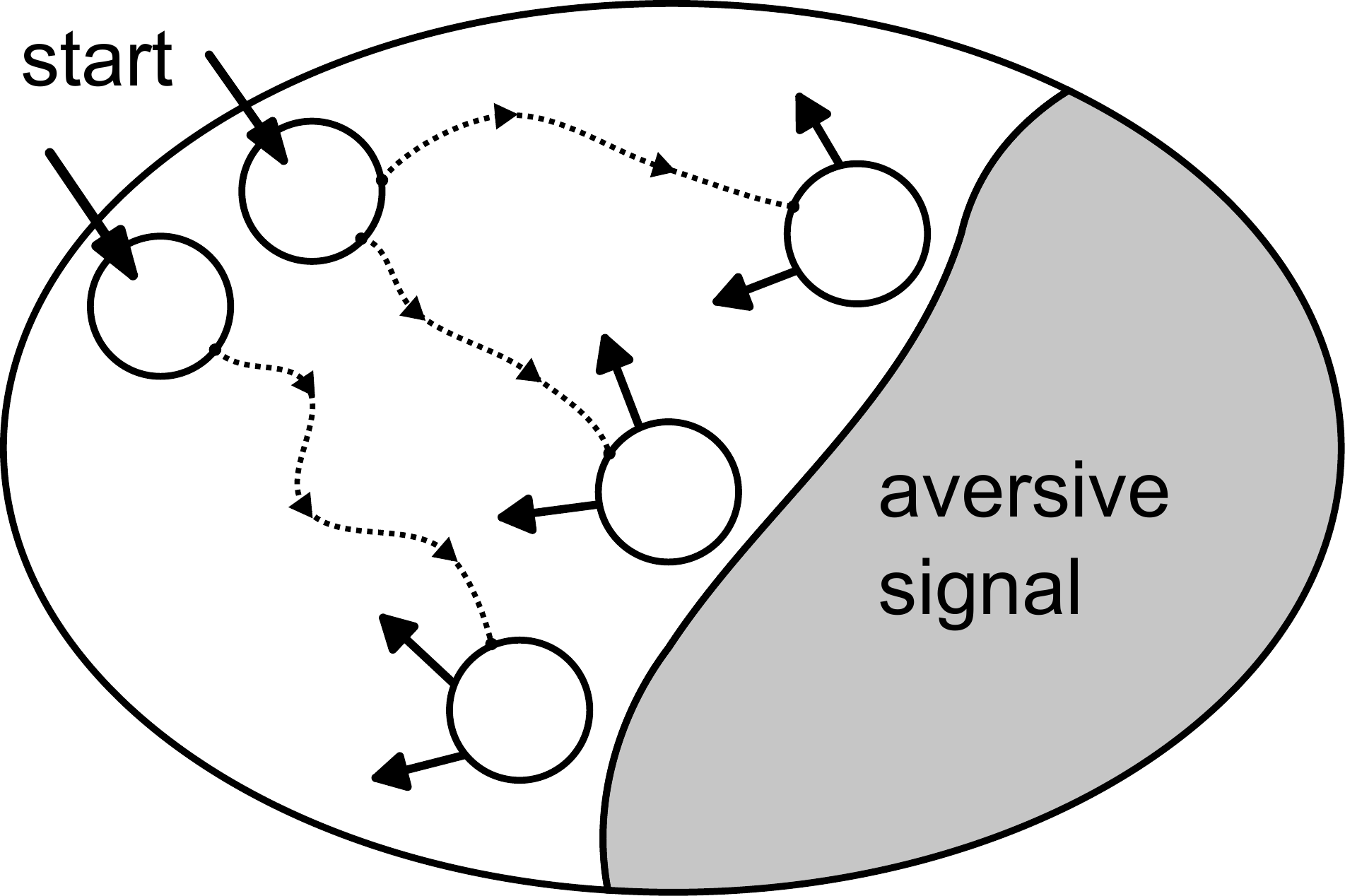}
    \end{center}
    \caption{Tasks contain aversive signals that should be avoided. In general, we allow multiple start states.}
    \label{fig:aversive}
\end{figure}

\paragraph*{Partial information and non-determinism}
First, we assume that the agent is given only partial information, as follows: each encountered task state is projected to a set of features. 
This is a propositional representation, where each feature is a true/false question posed about the state~\cite{vanotterlo_2009}.
The number of features determines the granularity by which states can be perceived by agents.    
Therefore, the behavior of the agent will be based on feature-action associations, and not on (direct) state-action associations.%
    \footnote{If the same feature is used by several states, this may be seen as a form of generalization over those states. In this paper, features are used directly, and we do not perform a (second) generalization step over the observed features.}
Each application can choose its own features and its own way of computing them.
Examples of features are: detected edges in images, impulses through sensors, temporal events over streams, AND-OR combinations thereof, etc.
In this paper, we assume that tasks have only a finite number of features, although there could still be many features.
Perhaps not surprisingly, theoretical investigations show how hard it is to solve tasks under partial information, see e.g.~\cite{lusena_2001,roy_2005,chatterjee_2015}.%
    \footnote{Part of the motivation for this paper is also to reason about the feature design for solving tasks. An example is given in Section~\ref{sec:grid}.}

Second, we allow tasks to be non-deterministic. This means that the effect of some action-applications to states can not be predicted. In this paper, we assume that non-determinism is an inherent property of tasks. Although partial information also limits the reasoning within the agent, and therefore generally prevents accurate predictions, it remains a separate assumption to allow tasks themselves to be non-deterministic. For example, one may consider tasks in which features actually provide complete information, and where the agent could still struggle with non-determinism.%
    \footnote{We will see an example of this situation a bit later in the Introduction.}

\paragraph*{Strategies}

The focus of this paper to understand agents based on their behavior in tasks, which could be a useful way to understand intelligence in general~\cite{pfeifer-scheier_1999}.
As remarked earlier, in this paper, agent behavior will be based on feature-action associations. Conceptually, we may think of the agent as having a set $\mem$ of possibly allowed feature-action pairs, and whenever the agent encounters a task state $\st$, the agent (thinks it) is allowed to perform all actions $\act$ for which there is a feature $\f$ observed in state $\st$ such that $(\f,\act)\in\mem$.
We also refer to $\mem$ as a policy.

We say that a set $\mem$ of feature-action pairs constitutes a \emph{strategy} for a start state if $\mem$ will never lead to an aversive signal when starting from that start state. 
We note that it is not always sufficient for the states near the aversive signals to steer away from them, because sometimes the agent may get trapped in a zone of the state space that does not immediately give aversive signals but from which it is impossible to reliably escape the aversive signals. The agent should avoid such zones, which could require that the agent anticipates aversive signals from very early on.


Our aim in this paper is to reason about the existence of such successful strategies for classes of tasks, and to discuss an algorithm to find such strategies automatically. A main challenge throughout this study is posed by the compression of state information into features and the uncontrollable outcomes due to non-determinism.

\paragraph*{Reward-based value estimation seems unsuitable}
Before presenting more details of our algorithm, we first argue that algorithms based on (numerical) reward-based value estimation do not always appear suitable for reliably finding problem-avoiding strategies.

On the theory side, convergence proofs of value estimation algorithms often require the learning step-size to decrease over time, see e.g.~\cite{jaakkola_1994, watkins-dayan_1992}. Intuitively, convergence of the estimated values arises because the decreasing learning step-size makes it harder and harder for the agent to learn as time progresses. However, we would like to avoid putting such limits on the agent, because: 
    (1) it is useful to also study more flexible agents because they might sometimes better describe real-world agents; 
    (2) in practice it might be difficult to estimate in what exact way the learning step-size should decrease; and, 
    (3) also in practice, there are no guarantees on what the estimated values will eventually be after a certain amount of time has passed, because the estimates depend strongly on random fluctuations during task performance (due to non-determinism).

In practice, a non-decreasing step-size, although potentially useful to model flexible agents that keep learning from their latest experiences~\cite{sutton-barto_1998}, can lead to problems of its own.
We illustrate this with the example task shown in Figure~\ref{fig:reward-task}. There is a start state $1$, and two actions $a$ and $b$ that lead back to state $1$. We assume complete information for now, i.e., state $1$ is presented completely to the agent as a feature with the same information, namely, the identity of state $1$.
Suppose that the state-action pair $(1,a)$ always gives reward $\Aplus$. But for the pair $(1,b)$, the reward could be either $\Bplus$ or $\Bminus$. Although the pair $(1,b)$ is clearly preferable over the pair $(1,a)$ in case of positive reward, there is the risk of incurring a strong negative reward. The negative reward represents an aversive signal. 
In the perspective of strategies from above, note that $(1,a)$ constitutes a strategy: constantly executing action $a$ in state $1$ leads to an avoidance of aversive signals forever.

\begin{figure}
\begin{center}
\begin{subfigure}[t]{.45\textwidth}   
    \begin{center}
    \includegraphics[height=2cm]{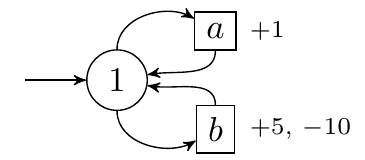}
    \end{center}
    \subcaption{Example task, with state $1$, and actions $a$ and $b$. Feedback to the agent is modeled as numerical reward.}
    \label{fig:reward-task}
\end{subfigure}
\quad
\begin{subfigure}[t]{.45\textwidth}
    \begin{center}
    \includegraphics[height=2cm]{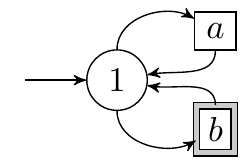}    
    \end{center}
    \subcaption{The same task as in Figure~\ref{fig:reward-task}, but now feedback is modeled as (boolean) aversive signals. Applying action $b$ results in an aversive signal, indicated by a box with double border.
    }
    \label{fig:aversive-task}
\end{subfigure}
\caption{An example task, represented in two ways. States and action applications are represented by circles and boxes respectively. Start states are indicated by an arrow without origin.}
\label{fig:intro-task}
\end{center}
\end{figure}

But the agent will not necessarily learn to avoid $(1,b)$ when a hidden task mechanism could periodically deceive the agent by issuing higher rewards under action $b$.
Concretely, let $n$ be a strictly positive natural number. To represent the outcome of action $b$, suppose that we constantly give reward $\Bplus$ during the first $n$ times $(1,b)$ is applied; the next $n$ times we give $\Bminus$; the following $n$ times we again give $\Bplus$, and so on. 
We call this the $n$-swap semantics.
For each outcome, the empirical probability would be $0.5$: indeed, the observed frequency of each outcome converges to $0.5$ as we perform more applications of action $b$. We can choose $n$ arbitrarily large; this does not change the empirical probability of each outcome.

Without the restriction on learning step-size, it seems that value estimation algorithms can get into trouble on the above setting because we can set $n$ so large that after a while the agent starts to believe that the outcome would remain fixed. 
For example, we could start with reward $\Bplus$ for the pair $(1,b)$ during the first $n$ applications, and the agent starts believing that the reward really is $\Bplus$. Then come the next $n$ applications, where we repeatedly give reward $\Bminus$, and the agent starts believing that the reward really is $\Bminus$. We can swap the two outcomes forever, each for a period of $n$ applications, and the agent will never make up its mind about the behavior of action $b$ in state $1$.
This effect is illustrated in Figure~\ref{fig:reward-pattern}.

\begin{figure}
\begin{center}
\begin{subfigure}[t]{.45\textwidth}
    \includegraphics[width=1\textwidth]{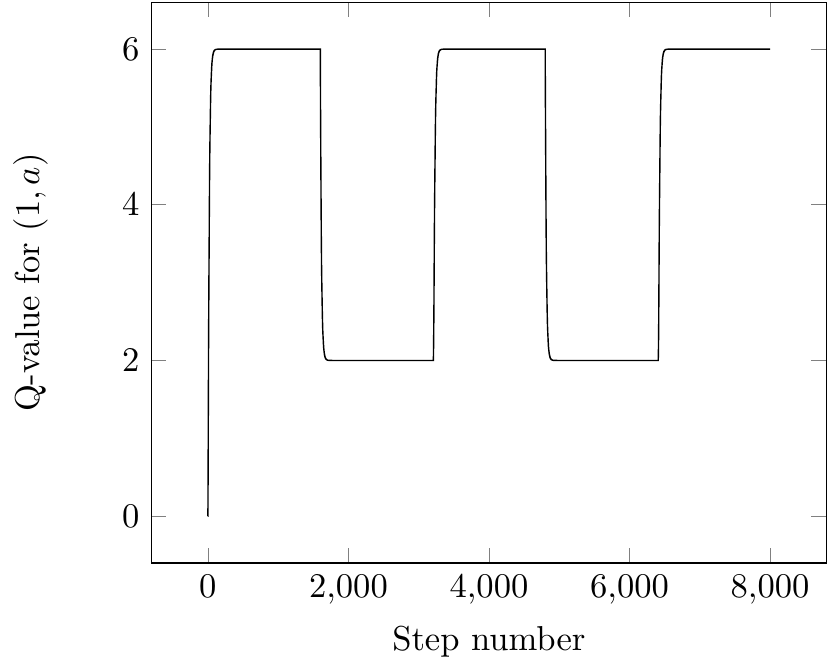}
    \subcaption{Q-values for pair $(1,a)$.}
\end{subfigure}
\quad
\begin{subfigure}[t]{.45\textwidth}
    \includegraphics[width=1\textwidth]{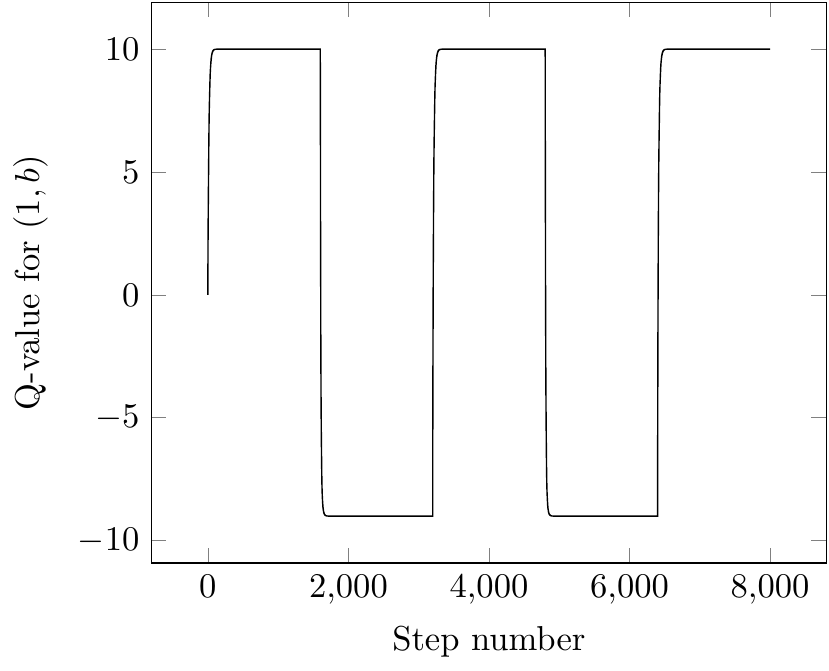}
    \subcaption{Q-values for pair $(1,b)$.}
\end{subfigure}
\end{center}
\caption{We have simulated the Q-learning algorithm~\cite{watkins_1989,watkins-dayan_1992} on the example task shown in Figure~\ref{fig:reward-task}, by alternating actions $a$ and $b$, but with constant step-size $\alpha=0.25$ (and discounting factor $\gamma=0.5$). 
The resulting value estimates for the state-action pairs $(1,a)$ and $(1,b)$ are plotted against time.
The outcome of $(1,b)$ was either $\Bplus$ or $\Bminus$, as determined by the $n$-swap semantics where $n=800$; this semantics is relative to the applications of $b$, and not relative to the global time steps.}
\label{fig:reward-pattern}
\end{figure}

Although the above example is very simple, real-world tasks could still exhibit problems similar to the $n$-swap semantics. Even if such problems are identified and understood, perhaps there are no good solutions for them as the problems might be outside the range of control for the agent.
In this paper we would like to learn to avoid the aversive signals forever, even under quite adversary semantics of tasks like the $n$-swap semantics.

\paragraph*{Avoidance learning}

In the example of Figure~\ref{fig:reward-task}, we would like the agent to make up its mind more quickly that action $b$ leads to aversive signals. 
An idea is to let the agent (monotonically) increase its estimate of the value of a feature-action pair.
We should immediately observe, however, that this idea will not work when feedback remains to be modeled as reward, as in the example: once the outcome of $(1,b)$ is observed to be $\Bplus$; then remembering $\Bplus$ would lead to a preference of $(1,b)$ over $(1,a)$, causing a reward-seeking agent to  (accidentally) encounter negative rewards, i.e., aversive signals, indefinitely under $n$-swap semantics.

Fortunately, the idea of increasing estimates seems to work when feedback is modeled with aversive signals, even in face of non-determinism. %
Indeed, \citet{heger_1994} has previously proposed a learning algorithm in tasks where actions have numeric costs, representing aversive signals. By repeatedly remembering the highest observed cost for a state-action pair (with the $\max$-operator), and by choosing actions to minimize such costs, the agent learns to steer away from high costs. 
We would like to further elaborate this idea and how it relates to the notion of aversion-avoiding strategies mentioned above.

In our framework, we only explicitly model aversive signals, as boolean flags: the flag ``true'' would mean that an aversive signal is present. This leads to a framework that is conceptually neat and computationally efficient. 
Because a policy is either successful in avoiding aversive signals forever, or it is not, the choice of a boolean model aligns well with our motivation to study the relationship between learning and successful strategies.
To illustrate, the example of Figure~\ref{fig:reward-task} would be represented by Figure~\ref{fig:aversive-task}, where only the aversive signal is explicitly represented.
In general, the boolean flags will act like borders, to demarcate undesirable areas in the state space. Reward is now only implicit: by using a strategy, as mentioned earlier, the agent can stay away from the aversive signals forever.

In the above setting with explicit aversive signals, we describe an avoidance learning algorithm, called \aLearn, in which the agent repeatedly flags feature-action pairs that lead to aversive signals, or, as an effect thereof, to states for which all proposed actions are flagged (based on the observed features).
Intuitively, the flags indicate ``danger''.
On the example of Figure~\ref{fig:aversive-task}, \aLearn\ flags $(1,b)$ at the first occurrence of an aversive signal under action $b$; and, importantly, the strategy $(1,a)$ is never flagged.%
    \footnote{In our description of \aLearn\ (Section~\ref{sec:alg}), the flagged feature-action pairs are removed from the agent's memory.}
There is no second chance for changing the agents mind. This gives one of the strongest convergence notions in learning, namely, fixpoint convergence, where the agent eventually stops changing its mind about the outcome of actions.  

If there really is a strategy, avoidance learning will carve out a subset of good feature-action pairs from the mass of all feature-action pairs. This way, it seems that avoidance learning could be useful in making the agent eventually avoid aversive signals forever.
This provides the guaranteed agent performance we would like to better understand, as remarked at the beginning of the Introduction.

\paragraph*{Meaning of optimality}

In this paper we view an agent as being optimal if it can (learn to) avoid aversive signals forever. There is no explicit concept of reward.
Depending on the setting, or application, aversive signals can originate from diverse sources and together they can describe a very detailed image of what the agent is allowed to do, and what the agent is not allowed to do. One obtains a rich conceptual setting for reasoning about agent performance.

For example, suppose a robotic agent should learn to move boxes in a storehouse as fast as possible. We could emit an aversive signal when the robot goes beyond a (reasonable) time limit. Any other constraints, perhaps regarding battery usage, can be combined with the first constraint by adding more signals.

\paragraph*{Outline}
This paper is organized as follows.
    We discuss related work in Section~\ref{sec:relwork}.
    We introduce fundamental concepts like tasks, and strategies, in Section~\ref{sec:fund}.
    We present and analyze our avoidance learning algorithm in Section~\ref{sec:alg}.
    One of our results is that if there is a strategy for a start state then the algorithm will preserve the strategy. This mechanism can be used to materialize strategies if they exist.
    To better understand the nature of strategies, we prove the existence of strategies for a family of grid navigation tasks in Section~\ref{sec:grid}.

\section{Related Work}\label{sec:relwork}

The idea of avoiding aversive signals, or problems in general, is related to safe reinforcement learning~\cite{garcia_2015}. There, the goal is essentially to perform reinforcement learning, often based on approximation techniques for optimizing numerical reward, with the addition of avoiding certain problematic areas in the task state space. An example could be to train a robot for navigation tasks but while avoiding damage to the robot as much as possible.
In the current paper, feedback to the agent consists of the aversive signals. Reward becomes more implicit, as it lies in the avoidance of aversive signals.
Therefore, the viewpoint in this paper is that the agent is called optimal when it eventually succeeds in avoiding all aversive signals forever; there is no notion of optimizing reward.
The approach is related to a trend identified by~\citet{garcia_2015}, namely, the modification of the optimality criterion.%

The work by \citet{heger_1994} is closely related to our work.
The framework by \citet{heger_1994} provides feedback to the agent in the form of numerical cost signals, which, from the perspective of this paper, could be seen as aversive signals.
Similar to our $n$-swapping example in the Introduction (Figure~\ref{fig:reward-task}), \citet{heger_1994} provides other examples to motivate that estimation of expected values is not suitable for reliably deciding actions.
The learning algorithm proposed by \citet{heger_1994} maps each state-action pair to the worst outcome (or cost), by means of the $\max$-operator. By remembering the highest incurred cost for a state-action pair, the agent in some sense learns about ``walls'' in the state space that constrain its actions towards lower costs.    
The avoidance learning algorithm discussed in this paper (Section~\ref{sec:alg}) is similar in spirit to the one by \citet{heger_1994}. 
A deviation, however, is that we assume here a boolean interpretation of aversive signals, which leads to a neat and computationally efficient framework. 
We additionally identify the concept of strategies, under which the agent can avoid aversive signals forever. 
Our interest lies in understanding such avoidance strategies and their relationship to the avoidance learning algorithm. Moreover, we also focus on partial information, by letting the agent only observe features instead of full states.

\section{Fundamental Notions}\label{sec:fund}

\subsection{Tasks}
\label{sub:task}

For a set $X$, let $\powset X$ denote the powerset of $X$, i.e., the set of all subsets of $X$.
A task is a tuple $\task=\tasktup$ where
\begin{itemize}
    \item $\states$ is a nonempty set of states;    
    \item $\startstates\subseteq\states$ is a finite subset of start states;\footnote{The symbol $\startstates$ stands for ``begin''.}
    \item $\actions$ is a nonempty finite set of actions;
    \item $\features$ is a nonempty finite set of features;
    \item $\tr:\states\times\actions\to\powset\states$ is the transition function;
    \item $\ff:\states\to\powset\features$ is the feature function; and
    \item $\avs\subseteq\states\times\actions$ is the set of aversive signals,
\end{itemize}
where all states $\st\in\states$ are reachable in the sense that there is a sequence $\st_0,\act_0,\st_1,\act_1,\ldots,\st_n$ with $\st_0\in\startstates$, $\st_n=\st$, and $\st_{i}\in\tr(\st_{i-1},\act_{i-1})$ for each $i\in\intrange 1n$.

The function $\tr$ maps each pair $(\st,\act)\in\states\times\actions$ to a set of possible successor states, representing non-determinism.
The function $\ff$ associates a set of features to each state; an agent interacting with the task can only observe states through features and can therefore not directly observe states.
The meaning of a pair $(\st,\act)\in\avs$ is that the agent could witness an aversive signal when performing action $\act$ in state $\st$.%
    \footnote{In a fair task, if the agent would infinitely often perform action $\act$ in state $\st$, then the agent witnesses an aversive signal infinitely often during the application of $\act$ at state $\st$, but this signal could sometimes be omitted. See also Section~\ref{sub:fairness}.}
    
\begin{example}\label{ex:task-two-states}
     We define an example task $\task=\tasktup$ as follows:
        $\states=\set{1,2}$;
        $\startstates=\states$;
        $\actions=\set{a}$;
        $\features=\set{f,g}$;
        regarding $\tr$, we define
        \begin{align*}
            &\tr(1,a) = \set{1},\\
            &\tr(2,a) = \set{2};
        \end{align*}
        regarding $\ff$, we define
        \begin{align*}
            &\ff(1) = \set{f},\\
            &\ff(2) = \set{g};
        \end{align*}
        and, we define $\avs=\set{(1,a)}$.
    The task is depicted in Figure~\ref{fig:task-two-states}.
    \qed
\end{example}

\begin{figure}
    \begin{center}
    \includegraphics[width=0.3\textwidth]{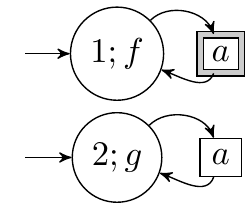}
    \end{center}
    \caption{The task from Example~\ref{ex:task-two-states}. The basic graphical notation is explained in Figure~\ref{fig:intro-task}. Inside the circles, we write the state identifier followed by a semicolon and the features of the state.}
    \label{fig:task-two-states}
\end{figure}

\begin{remark}[All features, one successor state]
    Note that the function $\ff$ maps each state to a set of features. Similarly, the function $\tr$ maps each state-action pair to a set of successor states.
    However, an agent interacts with each function in a different way, as follows. 
    For a state $\st$, we assume that an agent can always observe all features in $\ff(\st)$ simultaneously. 
    This way, the function $\ff$ may be viewed as being deterministic.   
    In contrast, for a state-action pair $(\st,\act)$, we select only one successor state from $\tr(\st,\act)$ to proceed with the task.
    
    The function $\ff$ remains deterministic throughout this paper. The framework still allows us to consider tasks in which the agent can sometimes observe a certain feature and sometimes it can not. Thereto we can define richer states, in which, say, the status of sensors is stored; if a state $\st$ says that a sensor is broken, then $\ff(\st)$ could omit the feature that would otherwise be generated by the sensor.
    \qed
\end{remark}

\begin{remark}[Modeling flexibility]
    Our definition of task resembles that of a standard Markov decision process~\cite{sutton-barto_1998}, but we have added features and aversive signals.
    There can be many features, actions, and start states. And we allow an infinite number of states.  
    \qed
\end{remark}

\subsection{Strategies}
\label{sub:strategy}

Since the agent may only see features, and not states directly, agent behavior has to be based on feature-action associations. 

Let $\task=\tasktup$ be a task.
A \emph{policy} for $\task$ is a total function $\pol:\features\to\powset{\actions}$. We allow features to be mapped to empty sets of actions.
If the task is understood from the context, for a state $\st\in\states$ we define
\[
    \prop\st\pol = \bigcup_{\mathlarger{\f\in\ff(\st)}}\pol(\f),
\]
i.e., $\prop\st\pol$ is the set of all actions that are proposed by the policy $\pol$ based on the features in $\st$.
We say that a state $\st$ is \emph{blocked in $\pol$} if $\prop\st\pol=\emptyset$, i.e., the policy does not propose actions for $\st$.

\begin{remark}[Features as actors]
    For a state $\st$, we do not view $\ff(\st)$ as an atomic signature to which actions should be associated. Instead, the definition of $\prop\st\pol$ indicates that each feature in $\ff(\st)$ may independently propose its own actions, regardless of what is proposed by other features. All proposed actions are collected into a set, by means of the union-operator.%
        \footnote{It would in general not be possible to replace this with an intersection-operator, because then there might be some unreliable features to which no action can be associated. Under an intersection-operator, such empty sets will destroy the proposals contributed by reliable features.}
    Therefore, features are little actors that become active at appropriate times and that suggest to the agent what actions are (supposedly) allowed.
    This viewpoint resembles the way that an individual neuron (or a small group of neurons) in the brain could represent a distinct concept and could be individually linked to actions~\cite{potjans_2011,fremaux_2013}.%
        \footnote{So, features constitute the agent's mind, like neurons constitute an organism's brain. Likely there are only finitely many neurons, so the assumption of having a finite number of features and a finite number of actions (or decisions) appears natural.}
    It is the goal of the learning algorithm (Section~\ref{sec:alg}) to remove feature-action associations that lead to aversive signals or, as a result of such removals, to blocked states.
    \qed
\end{remark}

We now consider the following definition:
\begin{definition}\label{def:strategy}
    
    A policy $\pol$ is called a \emph{strategy} for a start state $\st_0\in\startstates$ if
    \begin{enumerate}        

        \item \label{enu:strategy-start} $\prop{\st_0}\pol\neq\emptyset$;
        
        \item \label{enu:strategy-followup} $\forall\st\in\states$, $\forall\act\in\prop\st\pol$,
        \begin{enumerate}
            \item \label{enu:strategy-successor} $\forall\st'\in\tr(\st,\act)$ we have $\prop{\st'}\pol\neq\emptyset$; and,
            
            \item \label{enu:strategy-avs} $(\st,\act)\notin\avs$.
        \end{enumerate}
    \end{enumerate}
\qed
\end{definition}
In words: a policy is a strategy for a start state $\st_0$ if the policy acts upon $\st_0$; and, for any states upon which the policy acts, the reached successor states can also be acted upon, and the policy never causes aversive signals.
Intuitively, to use a strategy, for each encountered state $\st$ we first select some (arbritary) feature $\f\in\ff(\st)$ that satisfies $\pol(\f)\neq\emptyset$, and we subsequently select an arbitrary action $\act\in\pol(\f)$.

\begin{remark}[Global viewpoint]
    The definition of strategy demands properties in a global fashion, possibly also for states that would not be explored when strictly following the strategy. 
    This condition however ensures that learning algorithms can never have negative experiences when they perform actions suggested by the strategy; see Section~\ref{sec:alg}.
    
    Suppose $\pol$ is a strategy, and let $\f$ be a feature with $\pol(\f)\neq\emptyset$. Intuitively, the definition of strategy says that $\f$ is a reliable feature, in the sense that every time we see it, we may safely perform all actions in $\pol(\f)$, without the risk of encountering blocked states and aversive signals. This is related to the Markov assumption~\cite{sutton-barto_1998}, because we do not have to remember any features that were seen during previous time steps, and we may instead choose actions based on just $\f$ by itself.
    \qed
\end{remark}

\begin{example}\label{ex:partial-strategy}
    Consider the task from Example~\ref{ex:task-two-states}.
    There is no strategy for start state $1$, but there is a strategy $\pol$ for start state $2$ defined as: $\pol(f)=\emptyset$ and $\pol(g)=\set{a}$.
    \qed
\end{example}

The following property illustrates that strategies are resilient to adding new features. In practical applications, this means that the addition of new kinds of features will not destroy previously existing strategies.%
    \footnote{For example, in a robotics application, new features could be the result of adding new sensor types to the robot.}
\begin{proposition}\label{result:features}
    Let $\taskX 1 = \tasktupX 1$ be a task.
    Let $\view$ be a set of features that is disjoint from $\featuresX 1$.
    Let $\taskX 2 = \tasktupX 2$ be another task that is almost the same as $\taskX 1$ except that $\featuresX 2=\featuresX 1\cup\view$ and for each state $\st$ the constraint $\ffX 2(\st)\cap\featuresX 1=\ffX 1(\st)$ holds.%
        \footnote{This means that $\taskX 2$ uses the features of $\featuresX 1$ in the same way as $\taskX 1$.}
    Let $\st_0$ be a start state, and suppose that a policy $\pol:\featuresX 1\to\powset{\actionsX 1}$ is a strategy for $\st_0$ in $\taskX 1$. Then $\pol$ is also a strategy for $\st_0$ in $\taskX 2$.
\end{proposition}
\begin{proof}
    We show that the conditions of strategy in Definition~\ref{def:strategy} are satisfied for $\pol$ in $\taskX 2$.
    To better show which task is involved, for a state $\st$ and a task index $i\in\set{1,2}$, we write 
    \[
        \propX i\st\pol=\bigcup_{\f\in\ffX i(\st)}\pol(\f).
    \]
    When $\pol$ is used in $\taskX 2$, we assume $\pol(\f)=\emptyset$ for each $\f\in\view$. Above we have also assumed that $\actionsX 1=\actionsX 2$.

    \paragraph*{Condition~\ref{enu:strategy-start}}
    Since $\pol$ is a strategy for $\st_0$ in $\taskX 1$, we have $\propX 1{\st_0}\pol\neq\emptyset$. 
    This implies there is some $\f\in\ffX 1(\st_0)$ with $\pol(\f)\neq\emptyset$.
    Since $\ffX 1(\st_0)\subseteq\ffX 2(\st_0)$ by assumption, we obtain $\propX 2{\st_0}\pol\neq\emptyset$.

    \paragraph*{Condition~\ref{enu:strategy-followup}}
    Let $\st$ be a state and assume there is some action $\act\in\propX 2\st\pol$.
    First we argue that $\act\in\propX 1\st\pol$.
    There must be a feature $\f\in\ffX 2(\st)$ with $\act\in\pol(\f)$. But since $\pol$ only knows features in $\featuresX 1$, we have $\f\in\ffX 2(\st)\cap\featuresX 1=\ffX 1(\st)$. Hence, $\act\in\propX 1\st\pol$.
    
    We first handle Condition~\ref{enu:strategy-successor}.
    Let $\st'\in\tr(\st,\act)$.
    Because $\act\in\propX 1\st\pol$ and $\pol$ satisfies Condition~\ref{enu:strategy-successor} in $\taskX 1$, we know $\propX 1{\st'}\pol\neq\emptyset$. So there is a feature $\f'\in\ffX 1(\st')$ with $\pol(\f')\neq\emptyset$. Since $\ffX 1(\st')\subseteq\ffX 2(\st')$, we know $\act\in\propX 2{\st'}\pol\neq\emptyset$, as desired.
    
    Now we handle Condition~\ref{enu:strategy-avs}.
    Because $\act\in\propX 1\st\pol$ and $\pol$ satisfies Condition~\ref{enu:strategy-avs} in $\taskX 1$, we know $(\st,\act)\notin\avsX 1$. Since $\avsX 2=\avsX 1$, we obtain $(\st,\act)\notin\avsX 2$, as desired.
\qedhere
\end{proof}

\section{Avoidance Learning}\label{sec:alg}

We present and study an avoidance learning algorithm, and its relationship to the concept of strategy introduced in Section~\ref{sub:strategy}.

\subsection{\aLearn\ Algorithm}

Algorithm~\ref{alg:global} is an avoidance learning algorithm. The algorithm describes how the agent interacts with the task, and how feature-action combinations are forgotten as the direct or indirect result of aversive signals. Some aspects of the interaction are not under control of the agent, in particular how a successor state is chosen by means of function $\tr$, and how features are derived from states by means of function $\ff$. We now provide more discussion of the algorithm. Henceforth, we will refer to Algorithm~\ref{alg:global} as \aLearn.

\begin{algorithm}[h]
\BlankLine{}
\KwData{Task $\task=\tasktup$.}

\BlankLine{}
$\mem := \features\times\actions$\;
\label{line:init-mem}

$\st$ := choose from $\startstates$\;
\label{line:init-state}

\BlankLine{}
\If{$\prop\st\mem=\emptyset$}{
    \label{line:start-fail}
    request restart (see below)\;    
}

\BlankLine{}
\While{true}{
\label{line:loop}
    
    \BlankLine{}
    \If{\textnormal{restart requested}}{
        \label{line:desired-restart}
        go to \refline{line:init-state}\;
    }
    
    \BlankLine{}
    $\act$ := choose from $\prop\st\mem$\;
    \label{line:action}

    \BlankLine{}
    $\st'$ := choose from $\tr(\st,\act)$\;    
    \label{line:succ-state}
    
    \BlankLine{}    
    \If{$(\st,\act)\in\avs$ \textnormal{or} $\prop{\st'}\mem=\emptyset$}{
        \label{line:feedback}
       $\mem := \mem \setminus (\ff(\st)\times\set{\act})$\;
       \label{line:exclude}
       request restart\;
       \label{line:fail}
    }    
    
    \BlankLine{}   
    $\st$ := $\st'$\;   
    \label{line:continue}    
}
\caption{Avoidance learning (\aLearn)}
\label{alg:global}
\end{algorithm}

The essential product of \aLearn\ is a set $\mem\subseteq\features\times\actions$ that represents the allowed feature-action pairs; the symbol $\mem$ stands for ``possibilities''.
At any time, the set $\mem$ uniquely defines a policy $\pol$ as follows: for each $\f\in\features$, we define $\pol(\f)=\set{\act\in\actions\mid(\f,\act)\in\mem}$.
Regarding notation, for any state $\st$, we write $\prop\st\mem$ to denote the set $\prop\st\pol$ of proposed actions, where $\pol$ is the unique policy defined by $\mem$.

We now explain the steps of \aLearn\ in more detail.
\begin{itemize}
    \item 
    \refline{line:init-mem} initializes $\mem$ with all feature-action pairs. We will gradually remove pairs if they lead to $\avs$ or to blocked states (that are created by removals of the first kind).
    
    \item
    \refline{line:init-state} selects a random start state. The control flow of the algorithm is redirected here each time we want to restart the task. But we never re-initialize $\mem$.
    
    Task restarts may be requested by \aLearn\ itself (see below), or externally by the training framework in which \aLearn\ is running.
    
    \item 
    \refline{line:start-fail} requests a task restart in case the chosen start state is blocked.
    This allows more exploration from the other start states. As we will see later in Theorem~\ref{theo:learn}(\ref{enu:theo-learn-preserve}), if no actions remain for a start state then this start state has no strategy.
    
    \item At \refline{line:loop}, the algorithm enters a learning loop. The loop is only exited to satisfy task restart requests, at \refline{line:desired-restart}.

    \item At \refline{line:action}, we choose an action $\act$ to apply to current state $\st$ based on the set $\prop\st\mem$ of still allowed actions.
    At \refline{line:succ-state}, we are subsequently given a successor state $\st'$, chosen arbitrarily from $\tr(\st,\act)$.
    
    \item Next, at \refline{line:feedback}, we check whether we have encountered $\avs$ or if successor state $\st'$ is blocked. In either case we exclude from $\mem$ the feature-action pairs that caused us to apply action $\act$ in state $\st$ (\refline{line:exclude}), and we restart the task (\refline{line:fail}).

    \item If we do not encounter $\avs$ and state $\st'$ is not blocked, then we proceed with the while loop (\refline{line:continue}).
\end{itemize}

Note that in general there are multiple runs of \aLearn\ on a task, because of the choice on action selection and the choice on successor state.
Each run of \aLearn\ is infinitely long. Nonetheless, there is always an eventual fixpoint on the set $\mem$ because after the initialization we only remove feature-action pairs. There are only a finite number of possible feature-action pairs, although there could be many.    
When the run is clear from the context, we write $\memfix$ to denote the fixpoint of $\mem$ obtained in that run.

For conceptual convenience, we can divide each run of \aLearn\ into trials by using the task restarts as dividers: whenever we execute \refline{line:init-state}, the previous trial ends and the next trial begins. Each trial is thus a sequence $\st_0,\act_0,\st_1,\act_1,\ldots,\st_n$, where $\st_0$ is a start state, $\st_n$ is the last state of the trial, and $\st_i\in\tr(\st_{i-1},\act_{i-1})$ for each $i\in\intrange 1n$.%
    \footnote{We only use explicit task restarts (\refline{line:init-state}) to divide runs into trials, and not the encounter of start states. This means that in principle we allow $\st_i\in\startstates$ for some or all $i\in\intrange 1n$.}

\begin{remark}[No stopping condition]
There is no stopping condition in the algorithm because in general we may not be able to detect when the agent has explored the task sufficiently to be successful at avoiding aversive signals. \qed
\end{remark}

\begin{remark}[Greediness]\label{remark:greedy}
    We would like to emphasize that \aLearn\ is always greedy in avoiding $\avs$. 
    This is an important deviation from the $\epsilon$-greedy exploration principle~\cite{sutton-barto_1998}, where at each time step the agent chooses a random action with small probability $\epsilon\in[0,1]$.
    We do not use that mechanism here because otherwise the agent keeps running the risk of encountering aversive signals~\cite{garcia_2015}.
    \qed
\end{remark}

\begin{remark}[Internal task restarts]
    The reason for requesting a task restart at \refline{line:fail} is that sometimes the agent could become stuck in a zone of the state space where there are only blocked states or aversive signals. In that case, if we want the agent to start removing feature-action pairs to prevent future aversive signals, we should first transport the agent to a zone in the state space without blocked states and aversive signals.
    For example, in a robot navigation problem, the robot could learn to avoid pits, but once it enters a pit it can perhaps not reliably escape without the help of an external supervisor.
    \qed
\end{remark}

\begin{remark}[Memory efficiency]
\label{remark:memory}
Algorithm~\ref{alg:global} explicitly stores the allowed feature-action pairs in a set $\mem$. This is an intuitive perspective for the theory developed in this paper.
However, in practice it may sometimes be more efficient to store the opposite information, namely, the removed feature-action pairs. This way all allowed feature-action pairs can still be uniquely recovered.
Using the analogy of a planar map, where aversive signals are borders between neutral zones on the one hand and undesirable zones on the other hand, there could be a decreased memory usage in storing only the border (i.e., the removed feature-action pairs) if the borders are simple shapes instead of irregular shapes with many protrusions.
\qed
\end{remark}

\subsection{Results}

The following theorem helps to understand what \aLearn\ computes.
\begin{theorem}\label{theo:learn}
    For all tasks $\task=\tasktup$,
    for each $\st_0\in\startstates$, for each run of \aLearn, where $\memfix$ denotes the fixpoint,
    \begin{enumerate}%
        \item \label{enu:theo-learn-preserve} if there is a strategy for $\st_0$ then $\prop{\st_0}\memfix\neq\emptyset$.%
            
        \item \label{enu:theo-learn-discover} if $\prop{\st_0}\memfix\neq\emptyset$ then every trial for $\st_0$ after the fixpoint avoids blocked states and $\avs$.
    \end{enumerate}    
\end{theorem}
\begin{proof}
We consider the two properties separately.

    \paragraph*{Property~\ref{enu:theo-learn-preserve}}
    
        Suppose there is a strategy $\pol$ for $\st_0$. We show that the feature-action pairs of $\pol$ are preserved in $\memfix$, so that $\prop{\st_0}\pol\neq\emptyset$ would imply $\prop{\st_0}\memfix\neq\emptyset$.        
        Towards a contradiction, suppose that \aLearn\ removes a pair $(\f,\act)$ from $\mem$ where $\act\in\pol(\f)$; let $(\f,\act)$ be the first such pair that is removed.
        The removal has happened as follows: we reach a state $\st$ with $\f\in\ff(\st)$ and we perform $\act$, and either the successor state $\st'\in\tr(\st,\act)$ is blocked or we receive an aversive signal. We discuss each case in turn.
        
        Let $\mem$ denote the remaining feature-action pairs just before we remove $(\f,\act)$. 
        Note that $\act\in\pol(\f)$ and $\f\in\ff(\st)$ together imply $\act\in\prop\st\pol$.
        \begin{itemize}
            \item  %
            Suppose that $\st'$ is blocked.
            Since $\pol$ is a strategy, by condition \ref{enu:strategy-successor} of Definition~\ref{def:strategy}, we have assumed $\prop{\st'}\pol\neq\emptyset$. So, there is a feature $\f'\in\ff(\st')$ and an action $\act'\in\pol(\f')$.
            Since $(\f,\act)$ is the first pair of $\pol$ that is removed, we still have $(\f',\act')\in\mem$.
            But then $\act'\in\prop{\st'}\mem$, and $\st'$ is actually not blocked; we have found a contradiction.
            
            \item Suppose that an aversive signal was received when applying $\act$ to $\st$, which implies $(\st,\act)\in\avs$. This immediately contradicts the assumption that $\pol$ satisfies condition~\ref{enu:strategy-avs} of Definition~\ref{def:strategy}.
       \end{itemize}
                
    \paragraph*{Property~\ref{enu:theo-learn-discover}}
    
       Suppose $\prop{\st_0}\memfix\neq\emptyset$.            
        Towards a contradiction, suppose that after the fixpoint there is a trial for start state $\st_0$ where we encounter a state $\st$ and we perform an action $\act$ such that either the successor state is blocked or we receive an aversive signal. Suppose we conceptually halt the offending trial at the first encountered problem. We have followed a path:
        \[
            \st_0
                \jump{\act_0}
            \st_1
                \jump{\act_1}
            \ldots
            \st_{n-1}
                \jump{\act_{n-1}}
            \st_n=\st
                \jump{\act_n=\act}
            \st',
        \]
        for some $\st'\in\tr(\st,\act)$.
        We have $\act_i\in\prop{\st_i}\memfix$ for each $i\in\intrange 0n$.   
        We note in particular that $\act\in\prop\st\memfix$.       
        Next we distinguish two cases, depending on the type of problem.
        \begin{itemize}
            \item Suppose that $\prop{\st'}\memfix=\emptyset$.
            Then \aLearn\ now removes $\ff(\st)\times\set{\act}$ from $\mem$. But then we will no longer propose action $\act$ for state $\st$, which was previously allowed by the fixpoint. Then $\memfix$ would be an invalid fixpoint, which is a contradiction.
        
            \item Suppose that an aversive signal is received when applying $\act$ to $\st$, which implies $(\st,\act)\in\avs$. We make a similar reasoning as in the previous case: \aLearn\ removes $\ff(\st)\times\set{\act}$ from $\mem$. Again $\memfix$ would be an invalid fixpoint.
        \end{itemize}     
\end{proof}

\begin{remark}[Strategies and eventual success]
Suppose that a task has a strategy for each start state. In that case, Theorem~\ref{theo:learn} tells us that every run of \aLearn\ will eventually avoid blocked states and aversive signals. The agent therefore makes a transition from first discovering the strategies to later exploiting the strategies.

The opposite is not necessarily true: there are tasks for which exist runs that eventually avoid blocked states and aversive signals, but without there being a strategy in the sense of Definition~\ref{def:strategy}. This is illustrated by the task in Figure~\ref{fig:aversive-task-long}. Consider a run where the first application of action $b$ in state $1$ results in an aversive signal, and after which we immediately restart the task. In that run, there is no further exploration to state $2$, which causes $(f,a)\in\memfix$; hence, $\prop{1}\memfix=\set{a}\neq\emptyset$.
    However, note that if the internal restart request at \refline{line:fail} of Algorithm~\ref{alg:global} would sometimes not be handled immediately, but a few steps later, then some runs will not preserve the pair $(f,a)$.
\qed
\end{remark}

\begin{figure}[h]
    \begin{center}
    \includegraphics[width=0.4\textwidth]{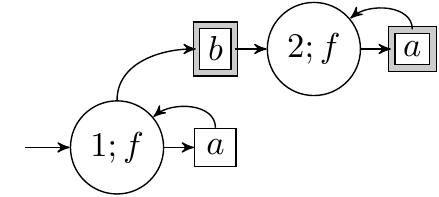}
    \end{center}
    \caption{An example task that has no strategy for the start state $1$. But there is a run of \aLearn\ in which the feature-action pair $(f,a)$ is preserved. The graphical notation is explained in Figure~\ref{fig:task-two-states}.}
    \label{fig:aversive-task-long}
\end{figure}

\begin{remark}[Usage of \aLearn]
    The insights of Theorem~\ref{theo:learn} could be used as follows.
    First, although proving that a strategy exists helps in understanding guarantees on the agent performance, programming the strategy by hand could be tedious and time-consuming. So, Property~\ref{enu:theo-learn-preserve} could be used to materialize strategies once they are proven to exist.
    
    Second, if one does not know whether a strategy exists, Property~\ref{enu:theo-learn-discover} could be used to perform a preliminary search for strategies. Although the discovered strategies might not be easily interpreted, they could serve as inspiration for a theoretical study of strategies for the tasks at hand.     
    A practical consideration, however, is that it might not be possible to efficiently detect the fixpoint, i.e., typically one does not know if a fixpoint has been reached when \aLearn\ has not removed feature-action pairs for a while.
\qed
\end{remark}

\subsection{Fairness}
\label{sub:fairness}

So far we have silently allowed all possible runs of \aLearn. For example, we did not explicitly demand that the agent actually must receive an aversive signal when applying an action $\act$ to a state $\st$ where $(\st,\act)\in\avs$. The aversive signal could also be omitted. This brings us to the topic of fairness~\cite{francez_1986}.

Intuitively, for this paper, fairness would mean that there is sufficient exploration of the task.
A practical application of \aLearn\ (Algorithm~\ref{alg:global}) could take the following fairness assumptions into account:
\begin{itemize}
    \item if we execute \refline{line:init-state} infinitely often then we choose each start state infinitely often;
    
    \item to fully learn the task from each start state, we infinitely often issue external task restarts at \refline{line:desired-restart}; those restarts are not requested by \aLearn\ itself;

    \item at \refline{line:action}, if we encounter the same pair of a state $\st$ and set $\mem$ infinitely often then we choose each action $\act\in\prop\st\mem$ infinitely often;%
            
    \item at \refline{line:succ-state}, if we apply action $\act$ infinitely often to state $\st$ then each successor state $\st'\in\tr(\st,\act)$ is visited infinitely often from an application of $\act$ to $\st$;

    \item at \refline{line:feedback}, if we perform action $\act$ in state $\st$ infinitely often, where $(\st,\act)\in\avs$, then the agent should infinitely often receive an aversive signal when applying $\act$ to $\st$; 
\end{itemize}
The only aspect of fairness that can be directly influenced by the agent itself, is the action selection at \refline{line:action}. For this purpose, a random number generator can be used to select random indices in an array-representation of the proposed actions.\footnote{One has to assume that the random number generator is fair in selecting all indices infinitely often if we let the system run forever.}
    
   
\begin{remark}[Fairness not required]
    Note that Theorem~\ref{theo:learn} also works for unfair runs.
    Every run has a fixpoint on $\mem$, whether the run is fair or not. But by exploring fewer states, or by issuing fewer aversive signals, an unfair run essentially makes it easier for the agent to avoid aversive signals.
    This way, some feature-action pairs could remain forever, even though a more fair exploration of the task could have removed them.
    
    Also, because the notion of strategy in Definition~\ref{def:strategy} is rather strong, it is not possible for a fair run or an unfair run to confront the agent with a situation that leads to the failure of a strategy. The agent will never be disappointed in the exploitation of the strategy.
    \qed
\end{remark}

\section{Simple Grid Navigation}\label{sec:grid}

We study a simple class of grid navigation problems.

\subsection{Definitions}\label{sub:grid-defs}

Let $\ints$ denote the set of integers. For any two points $p_1,p_2\in\ints\times\ints$, denoting $p_1=(x_1,y_1)$ and $p_2=(x_2,y_2)$, we recall the definition of $L_1$-distance between $p_1$ and $p_2$:
\[
    \mandist(p_1,p_2) = \abs{x_2 - x_1} + \abs{y_2 - y_1}.
\]

A \emph{simple grid navigation problem} is a quintuple $\grid=(\gridwidth,\gridheight,\gridstarts,\allowbreak\gridtargets,\allowbreak\gridtime)$, where
\begin{itemize}
    \item $\gridwidth\in\nat$ and $\gridheight\in\nat$ are the dimensions of a terrain; 
    \item $\gridstarts\subseteq\intrange 0\gridwidth\times\intrange 0\gridheight$ is a set of start locations;
    \item $\gridtargets\subseteq\intrange 0\gridwidth\times\intrange 0\gridheight$ is a set of possible target locations; and,
    \item $\gridtime\in\nat$ is a time limit,
\end{itemize}
with the following assumptions,
\begin{itemize}    
    \item $\forall p,q\in\gridtargets$, we assume $\mandist(p,q)< \gridtime$; and,
    \item $\forall p\in\gridstarts, \forall q\in\gridtargets$, we assume $\mandist(p,q)< \gridtime$.
\end{itemize}

The intuition is that at the beginning of a session we select a start location $p\in\gridstarts$ and an initial active target location $q\in\gridtargets$ and we should navigate from $p$ to $q$ within time $\gridtime$. Whenever we reach the active target location $q$ we choose another target location $q'\in\gridtargets$ and we should now navigate from $q$ to $q'$ within time $\gridtime$. This relocation of the active target may be repeated an arbitrary number of times. But at any moment we may also begin a new session, in which we again choose a start location and initial target location. There are infinitely many sessions.
The available actions are: left, right, up, down, left-up, left-down, right-up, right-down, and wait.
Importantly: failure to respect the time $\gridtime$ results in an aversive signal; we aim to eventually avoid such aversive signals.

For a location $(x,y)\in\nat\times\nat$ and an action $a$, we now define the possible successor locations that result from the application of $a$ to $(x,y)$; we denote this set as $\gridmove(x,y,a)$. A set of multiple possible successors is used to represent non-determinism. An empty set of of successors is used to say that the action would lead outside the considered terrain.
We assume the following actions to be deterministic: left, right, up, and down. The other, ``diagonal'', actions are non-deterministic. 
For example, for each $(x,y)\in\ints\times\ints$,
\[
    \gridmove(x,y,\text{left}) = \set{(x-1,y)},
\]
\[
    \gridmove(x,y,\text{left-up}) = \set{(x-1,y-1), (x-1,y), (x,y-1)}, 
\]and,
\[
    \gridmove(x,y,\text{wait}) = \set{(x,y)}.
\]
We make the assumption that the direction of the positive Y-axis corresponds to ``downward''.

We now define the task structure $\taskof\grid=\tasktup$ that corresponds to the above grid problem $\grid$. Here it will be convenient to view states and features as structured objects, with components; for an object $x$ with a component $y$, we write $x\g y$ to access the component.

\begin{itemize}
    \item the set $\states$ consists of all triples $\st$ with components \emph{agent}, \emph{target}, and \emph{time}, satisfying the following constraints: 
    $\st\g{agent}$ and $\st\g{target}$ are both in the set $\intrange 0\gridwidth\times\intrange 0\gridheight$, and 
    $\st\g{time}\in\intrange 0\gridtime$;
    
    \item the set $\startstates$ consists of those states $\st$ where $\st\g{agent}\in\gridstarts$, $\st\g{target}\in\gridtargets$, and $\st\g{time}=\gridtime$;    
    
    \item $\actions=\set{\text{left, right, up, down, left-up, left-down, right-up, right-down, wait}}$;
    
    \item the set $\features$ consists of all pairs $\f$ with components \emph{offset} and \emph{time}, satisfying the constraints: $\f\g{offset}\in\intrange{-\gridwidth}\gridwidth\times\intrange{-\gridheight}\gridheight$ and $\f\g{time}\in\intrange 0\gridtime$;%
        \footnote{These features represent an agent-centric perspective, in which the relative offset to the target is stored (see below).}
    
    \item the transition function $\tr$ is described by Algorithm~\ref{alg:grid-trans}; for a state $\st\in\states$ and action $\act\in\actions$, the set $\tr(\st,\act)$ consists of all states that could possibly be returned by Algorithm~\ref{alg:grid-trans} upon receiving input $(\st,\act)$;

    \item regarding $\ff$, for each $\st\in\states$, we define $\ff(\st)=\set{\f}$ where $\f\in\features$ is the single feature for which $\f\g{offset}=\st\g{target}-\st\g{agent}$ and $\f\g{time}=\st\g{time}$;
    and,
    
    \item $\avs=\set{(\st,\act)\in\states\times\actions\mid \st\g{time}=0}$.
\end{itemize}


\begin{algorithm}
\caption{Action application for grid navigation (Section~\ref{sec:grid})}\label{alg:grid-trans}

\SetKwInOut{Input}{input}
\SetKwInOut{Output}{output}
\LinesNumbered{}

\BlankLine{}
\Input{(1) current state $\st\in\states$\\ 
       (2) action $\act\in\actions$}
\BlankLine{}
\Output{successor state $\st'\in\states$}

\BlankLine{}
$(x_1,y_1)$ := $\st\g{agent}$\;
$(x_2,y_2)$ := choose from $\gridmove(x_1,y_1,a)$\;

\BlankLine{}
\uIf{$(x_2,y_2)\in\intrange 0\gridwidth \times \intrange 0\gridheight$}{
    $\st'\g{agent}$ := $(x_2, y_2)$\;
}
\Else{
    $\st'\g{agent}$ := $(x_1,y_1)$\;
}

\BlankLine{}
\uIf{$\st'\g{agent} = \st\g{target}$}{
    \label{line:grid-reach-target}
    $\st'\g{target}$ := choose from $\gridtargets$\;
    $\st'\g{time}$ := $\gridtime$\;    
}
\Else{
    $\st'\g{target}$ := $\st\g{target}$\;
    $\st'\g{time}$ := $\max(0,\st\g{time}-1)$\;
}

\end{algorithm}

\subsection{Results}

\begin{proposition}\label{result:grid}
For each grid problem $\grid$, there is a strategy for each start state in $\taskof\grid$.
\end{proposition}
\begin{proof}
    Denote $\taskof\grid=\tasktup$. We define one policy that is a strategy for all start states.   
    
    First, we define an auxiliary set $\view\subseteq\features$ to consist of all features $\f$ for which
    \[
        \mannorm{\f\g{offset}}<\f\g{time},
    \]
    where $\mannorm{(x,y)}=\abs x+\abs y$ is the $L_1$-norm of a point $(x,y)$.
    Intuitively, such features indicate that the deterministic distance from the agent location to the target location -- where we only use the actions left, right, up, and down -- can be bridged within the remaining time.
    
    We now define a policy $\pol$.
    For all $\f\in\features\setminus\view$ we define $\pol(\f)=\emptyset$, and 
    for each $\f\in\view$, denoting $\f\g{offset}=(x,y)$, we define
    \[
        \pol(\f) = 
            \begin{cases}
            \set{\text{left}} & \text{if }x < 0;\\
            \set{\text{right}} & \text{if }x > 0;\\
            \set{\text{up}} & \text{if $x=0$ and $y < 0$};\\
            \set{\text{down}} & \text{if $x = 0$ and $y > 0$};\\
            \set{\text{wait}} & \text{$x=0$ and $y=0$}.
            \end{cases} 
    \]
    As mentioned earlier, we define downwards as the direction of the positive Y-axis.
    The case where $\pol(\f)=\set{\text{wait}}$ occurs when the agent is located at the target.%
        \footnote{Algorithm~\ref{alg:grid-trans} implies that the situation where $\f\g{offset}=(0,0)$ only occurs when the agent reaches some target location and the next target location is the same as the old target location.}
   
    Let $\st_0\in\startstates$. We show that $\pol$ is a strategy for $\st_0$, according to Definition~\ref{def:strategy}.
    
    \paragraph*{Condition~\ref{enu:strategy-start} of Definition~\ref{def:strategy}}
    We show that $\prop{\st_0}\pol\neq\emptyset$.
    By assumption on $\st_0$, we have $\st_0\g{agent}\in\gridstarts$, $\st_0\g{target}\in\gridtargets$, and $\st_0\g{time}=\gridtime$.
    By using the distance assumptions on locations in $\grid$, we obtain $\mandist(\st_0\g{agent},\st_0\g{target})<\gridtime$. 
    Letting $\f$ be the single feature in $\ff(\st_0)$, we see that $\mannorm{\f\g{offset}}<\gridtime=\f\g{time}$, which implies that $\f\in\view$.
    Hence $\pol(\f)\neq\emptyset$, which implies $\prop{\st_0}\pol\neq\emptyset$.
        
    \paragraph*{Condition~\ref{enu:strategy-successor} of Definition~\ref{def:strategy}}
    
    Let $\st\in\states$. Suppose there is some action $\act\in\prop\st\pol$. Let $\f$ denote the single feature of $\st$. We have $\act\in\pol(\f)$, which implies $\f\in\view$.
    
    Let $\st'\in\tr(\st,\act)$. We must show that $\prop{\st'}\pol\neq\emptyset$.
    Let $\f'$ be the single feature of $\st'$. We will show that $\f'\in\view$, which implies $\pol(\f')\neq\emptyset$, and further that $\prop{\st'}\pol\neq\emptyset$.
    Based on Algorithm~\ref{alg:grid-trans}, we reason about what has happened during the application of action $\act$ to state $\st$.
    \begin{itemize}
        \item Suppose the if-test at \refline{line:grid-reach-target} succeeds, i.e., the agent reaches the target location. Then $\st'\g{time}=\gridtime$, and 
        \begin{align*}
            \mandist(\st'\g{agent},\st'\g{target}) 
                &= \mandist(\st\g{target},\st'\g{target})\\
                &< \gridtime,    
        \end{align*}
        where we use the distance assumption between target locations.
        Overall, $\mannorm{\f'\g{offset}}<\f'\g{time}$; hence, $\f'\in\view$.
        
        \item Suppose the if-test at \refline{line:grid-reach-target} does not succeed, i.e., the agent did not yet reach the target location.
        It must be that $\act\neq\text{wait}$, because otherwise $\f\g{offset}=(0,0)$, which implies $\st'\g{agent}=\st\g{agent}=\st\g{target}$, and the test at \refline{line:grid-reach-target} would have succeeded (see previous case).
        So, $\act\in\set{\text{left, right, up, down}}$. 
        
        First, we observe that 
        \[
            \mandist(\st'\g{agent},\st'\g{target})<\mandist(\st\g{agent},\st\g{target}).
        \]
        Indeed, this property holds because (1) the locations $\st\g{agent}$ and $\st\g{target}=\st'\g{target}$ are inside the convex terrain; (2) the action $\act$ is given deterministic movement semantics (i.e., there is precisely one outcome), causing $\st'\g{agent}$ to be both inside the terrain and strictly closer to $\st\g{target}=\st'\g{target}$.
        
        Second, we also observe that
        \[
            \f'\g{time}=\f\g{time}-1,
        \]
        since $\st'\g{time}=\max(0,\st\g{time}-1)$ by definition and $\st\g{time}>0$ (which follows from $\f\in\view$).
        
        Overall, we may now write
        \begin{align*}
            \mannorm{\f'\g{offset}} 
                &< \mannorm{\f\g{offset}} \\
                &\leq \f\g{time} - 1\\
                &= \f'\g{time}.
        \end{align*}           
        In the second line we have used $\f\in\view$. We conclude that $\f'\in\view$.
    \end{itemize}
        
    \paragraph*{Condition~\ref{enu:strategy-avs} of Definition~\ref{def:strategy}}
    
    Let $\st\in\states$. Suppose there is some action $\act\in\prop\st\pol$, which implies that the single feature $\f$ of $\st$ must be in $\view$. 
    By definition of $\view$, we have $\mannorm{\f\g{offset}}<\f\g{time}$.
    Hence, $\st\g{time}>0$, which implies $(\st,\act)\notin\avs$, as desired.    
    \qedhere
\end{proof}

\begin{remark}[Richness in strategy]
    The policy defined in the proof of Proposition~\ref{result:grid} is in general not the maximal strategy, in the sense that the policy could be extended with more actions than currently specified. For instance, if the time limit is high then the agent can randomly wander around before it becomes sufficiently urgent to reach a target location.     
    The agent may also use the diagonal actions, like left-up, if the time limit is not violated under either of the three outcomes. 
    \qed
\end{remark}

\begin{remark}[Extendability]
    It is possible to extend the above setting of grid navigation to richer state representations, by including for example the locations of additional objects (that do not influence the agent). If this new information would be communicated to the agent with a set of features that is disjoint from the set of old features in Section~\ref{sub:grid-defs}, then Proposition~\ref{result:features} tells us that the strategy described in the proof of Proposition~\ref{result:grid} is still valid.     
    \qed
\end{remark}

\section{Conclusion and Further Work}\label{sec:conclusion}

We have used the notion of strategies to reason about the successful avoidance of aversive signals in tasks. We have shown that our avoidance learning algorithm always preserves those strategies.
Now we discuss some interesting topics for further work.

\paragraph*{Feature detectors}
In this paper we have considered a framework in which features are essentially black boxes, in the sense that we do not assume anything about the way that they are computed. Hence, we do not know how features are related to the task environment. 
It would be interesting to develop more detailed insights into how features can be designed, to ensure that strategies, or similarly successful policies, are possible.

In particular, it seems fascinating to explore possible connections between our framework and neuron-like models, where features would be represented by neurons or by small groups of neurons. It is currently an open question whether or not feature learning in the brain is a completely unsupervised process~\cite{fremaux_2015}, i.e., it is not known whether feature creation is influenced by rewarding or aversive signals. 
So, in a general theory, it might be valid to consider feature learning as a separate, unsupervised, module. This approach could lead to a conceptually simple framework of agent behavior and feature detection simultaneously.
Concretely, the approach could enable the results in this paper to be linked to various feature detector algorithms.

\paragraph*{The challenge of new features}
In this paper we have assumed that the set of features is fixed at the beginning of the learning process. This could be suitable for many applications, as there is no fixed limit on how many features there are, as long as there are finitely many.
But it seems intriguing to introduce new features while the agent is performing the task. In the technical approach of this paper, however, a newly inserted feature likely proposes wrong actions if we would initially associate all actions  to the feature. 
In general we still insist that aversive signals are avoided, and therefore the wrong actions need to be unlearned as soon as possible.

A way to soften the introduction of new features, could be to reintroduce reward into the framework. Concretely, a feature $\f$ may only propose an action $\act$ if the pair $(\f,\act)$ has been observed to be correlated to reward, either directly, or transitively by means of eligibility traces~\cite{sutton-barto_1998}. This idea introduces a threshold for proposing actions. Of course any feature-action pairs introduced in this way could still lead to aversive signals. For example, there could be spurious features (e.g. features that randomly appear) to which no actions should be linked, or perhaps the rewarding signals contradict the aversive signals, or some actions that give reward could also give aversive signals (as in the example of the Introduction). To resolve priority issues, one could view avoidance learning as having the highest precedence, where reward is used as a softer ranking mechanism on the allowed actions.

Possibly, an agent that keeps learning new features will keep making mistakes. How to cope with new features therefore seems a relevant question. The answers could perhaps also help to understand animal behavior and consciousness. Thereto one could consider other notions of success than the avoidance of aversive signals investigated in this paper.

\bibliographystyle{apalike} 


\end{document}